\documentclass{article}
\usepackage[preprint]{neurips_2020}

\usepackage{times}
\usepackage{soul}
\usepackage{url}
\usepackage[hidelinks]{hyperref}
\usepackage[utf8]{inputenc}
\usepackage[small]{caption}
\usepackage{graphicx}
\usepackage{amsmath}
\usepackage{amsthm}
\usepackage{booktabs}
\usepackage{algorithm}
\usepackage{algorithmic}
\usepackage{amssymb}
\usepackage{natbib}
\usepackage{xcolor}

\usepackage{subfig}

\DeclareMathOperator*{\argmax}{arg\,max}

\newtheorem{theorem}{Theorem}

\newtheorem{prop}{Proposition}
\newtheorem{lemma}{Lemma}

\title{A Consistent Diffusion-Based Algorithm for Semi-Supervised Classification on Graphs}

\author{
    Nathan de Lara \hspace{1cm} Thomas Bonald\\
    Télécom Paris, Institut Polytechnique de Paris   \\ 
    {\tt \{nathan.delara, thomas.bonald\}@telecom-paris.fr}
}

\begin{document}

\maketitle

\begin{abstract}
Semi-supervised classification on graphs aims at assigning labels to all nodes of a graph based on the  labels known for a few nodes, called the seeds. The most popular algorithm relies on the principle of heat diffusion, where the labels of the seeds are spread by thermo-conductance and  the temperature of each node is used as a score function for each label. Using a simple block model, we prove that this algorithm is not consistent unless the temperatures of the nodes are centered before classification. We show that this simple modification of the algorithm is enough to get significant performance gains on real data.
\end{abstract}


\section{Introduction}

Heat diffusion, describing by the evolution of temperature $T$ in an isotropic material, is governed by the heat equation:
\begin{equation}\label{eq:heat}
\dfrac{\partial T}{\partial t} = \alpha \Delta T,
\end{equation}
where $\alpha$  is the thermal conductivity of the material and $\Delta$ is the Laplace operator. 
In  steady state, this equation simplifies to $\Delta T = 0$ and the function $T$ is said to be {\it harmonic}. The Dirichlet problem consists in finding the equilibrium in the presence of  boundary conditions, that is when the temperature $T$ is fixed on the boundary of the region of interest.

The principle of heat diffusion has proved instrumental in graph mining, using a discrete version of the heat equation \eqref{eq:heat} \citep{kondor2002diffusion}.
  It has been applied for  many different tasks, including   pattern matching  \citep{thanou2017learning},  ranking   \citep{ma2008mining, ma2011mining},  embedding   \citep{donnat2018learning},    clustering \citep{tremblay2014graph} and  classification \citep{zhu2005semi,merkurjev2016semi,berberidis2018adadif, DBLP:journals/corr/abs-1902-06105}. In this paper, we focus on the classification task: given  labels known for some nodes of the graph, referred to as the {\it seeds}, how to infer the labels of the other nodes? The number of seeds is typically small compared to the total number of nodes (e.g., 1\%), hence the name of  {\it semi-supervised} classification.

The most popular algorithm for semi-supervised classification in graphs   is based on a discrete version of the Dirichlet problem, the seeds playing the role of  the ``boundary'' of the Dirichlet problem \citep{zhu2003semi}. Specifically, one Dirichlet problem is solved per label, setting the  temperature of the corresponding seeds at 1 and the temperature of the other seeds as 0. Each node is then assigned the label with the highest temperature over the different Dirichlet problems. 
In this paper, we prove on a simple block model that this {\it vanilla} algorithm is not consistent, unless the temperatures are {\it centered} before classification. Experiments show that this simple modification of the algorithm significantly improves classification scores  on real datasets.


The rest of this paper is organized as follows. In section \ref{sec:dir}, we introduce the Dirichlet problem on graphs. Section \ref{sec:algo} describes our algorithm for   node classification. The analysis showing the 
consistency of our algorithm on a simple block model is presented in 
section \ref{sec:model}.  Section \ref{sec:exp} presents the experiments and section \ref{sec:conc} concludes the paper.

\section{Dirichlet problem on graphs}
\label{sec:dir}

In this section, we introduce the Dirichlet problem on graphs 
and characterize the solution, which will be used in the analysis. 

\subsection{Heat equation}
\label{ssec:pb}

Consider a graph $G$ with $n$ nodes indexed from $1$ to $n$. Denote by $A$ its adjacency matrix. This is a symmetric, binary matrix.   Let $d = A\mathbf{1}$ be the degree vector, which is assumed positive, and $D = \text{diag}(d)$. The Laplacian matrix is defined by
$$
L = D - A.
$$

Now let  $S$ be some strict subset  of $\{1, \dots, n\}$ and assume that the temperature of each
node $i \in S$ is set at some fixed value $T_i$. 
We are interested in the evolution of the temperatures of the other nodes. Heat exchanges occur through each edge of the graph proportionally to the temperature difference between the corresponding nodes. Then,
$$
\forall i \notin S, \quad \dfrac{dT_i}{dt} = \underset{j=1}{\overset{n}{\sum}}A_{ij}(T_j - T_i),
$$
that is
$$
\forall i \notin S,\quad  \dfrac{dT_i}{dt} = -(LT)_i,
$$
where $T$ is the vector of temperatures. This is the heat equation in discrete space, where  $-L$ plays the role of the Laplace operator in \eqref{eq:heat}. At equilibrium, $T$ satisfies Laplace’s equation:
\begin{equation}
    \label{eq:laplace}
    \forall i \notin S,\quad (LT)_i = 0.
\end{equation}
We say that the vector $T$ is \textit{harmonic}. With the boundary conditions $T_i$ for all $i \in S$, this defines a Dirichlet problem in discrete space. Observe that Laplace's equation \eqref{eq:laplace} can be written equivalently:
\begin{equation}
    \label{eq:laplace2}
    \forall i \notin S,\quad T_i = (PT)_i,
\end{equation}
where $P =D^{-1}A$ is the transition matrix of the random walk in the graph.

\subsection{Solution to the Dirichlet problem}
\label{ssec:sol}

We now characterize the solution to the Dirichlet problem in discrete space. Without any loss of generality, we assume that nodes with unknown temperatures (i.e., not in $S$) are indexed from 1 to $n-s$ so that the vector of temperatures can be written
$$
T = \begin{bmatrix}X\\ Y\end{bmatrix},
$$
where $X$ is the unknown vector of temperatures at equilibrium, of dimension $n-s$. Writing 
the transition matrix in block form as
$$
P = \begin{bmatrix}Q & R \\ \cdot & \cdot\end{bmatrix},
$$
it follows from \eqref{eq:laplace2} that:
\begin{equation}
    \label{eq:linear}
    \quad X = QX + RY,
\end{equation}
so that:
\begin{equation}
    \label{eq:solinear}
    \quad X = (I-Q)^{-1}RY.
\end{equation}
Note that the inverse of the matrix $I-Q$ exists whenever  the graph is connected, which implies that  the matrix $Q$ is  sub-stochastic with spectral radius strictly less than 1 \citep{chung}.

The exact solution \eqref{eq:solinear} requires to solve a (potentially large) linear system. In practice, a very good approximation is provided by a few iterations of \eqref{eq:linear}, the rate of convergence  depending on the spectral radius of the  matrix $Q$. The small-world property of real graphs suggests that a few iterations are enough in practice  \citep{watts}. 
This will be confirmed by the experiments.

\subsection{Extensions}
\label{ssec:ext}

The results apply to
weighted graphs, with a positive weight assigned to each edge. This weight  can then be  interpreted as the \textit{thermal conductivity} of the edge in the diffusion process. 
Interestingly, the results also apply to directed graphs. Indeed, a directed graph $G$ of $n$ nodes, with  adjacency matrix $A$,   can be considered as a bipartite graph of $2n$ nodes, with adjacency matrix:
$$
\begin{bmatrix}
0 & A \\
A^T & 0
\end{bmatrix}
$$
The diffusion can be applied to this bipartite graph, which is undirected. Observe that 
each node of the directed graph $G$ is duplicated in the bipartite graph and is thus characterized by 2 temperatures, one as heat source (for outgoing edges) and one as heat destination (for incoming edges). 
It is not necessary for the directed graph to be  strongly connected; only the associate bipartite graph needs to be connected.

\section{Node classification algorithm}
\label{sec:algo}

In this section, we introduce a node classification algorithm based on  the Dirichlet problem.
The objective is  to infer the labels of all nodes given the  labels of a few nodes called the \textit{seeds}.
Our algorithm is a simple modification of the popular method proposed by \cite{zhu2003semi}. Specifically, we propose to center  temperatures  before classification.

\subsection{Binary classification}

When there are only two different labels, the classification can be done by solving  one Dirichlet problem.
The idea is to use the seeds with label 1 as hot sources, setting their temperature at $1$, and   the seeds with label 2 as cold sources, setting their temperature at $0$.
The solution to this Dirichlet problem gives temperatures between 0 and 1, as illustrated by Figure \ref{fig:karate}.

\begin{figure}[h]
    \centering
 \subfloat[Ground truth.]{
    \includegraphics[width=0.49\linewidth]{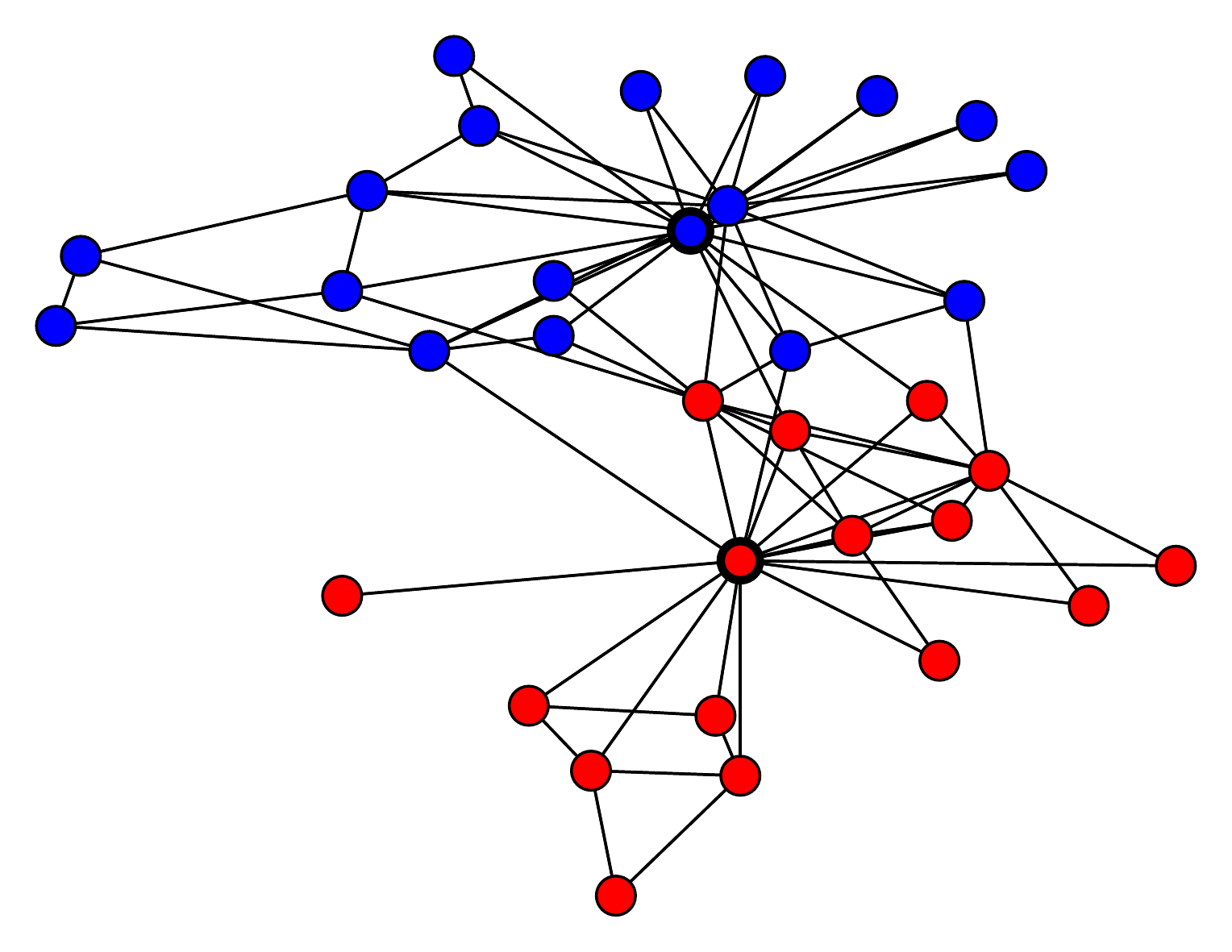}}
    \subfloat[Solution to the Dirichlet problem.]{ \includegraphics[width=0.49\linewidth]{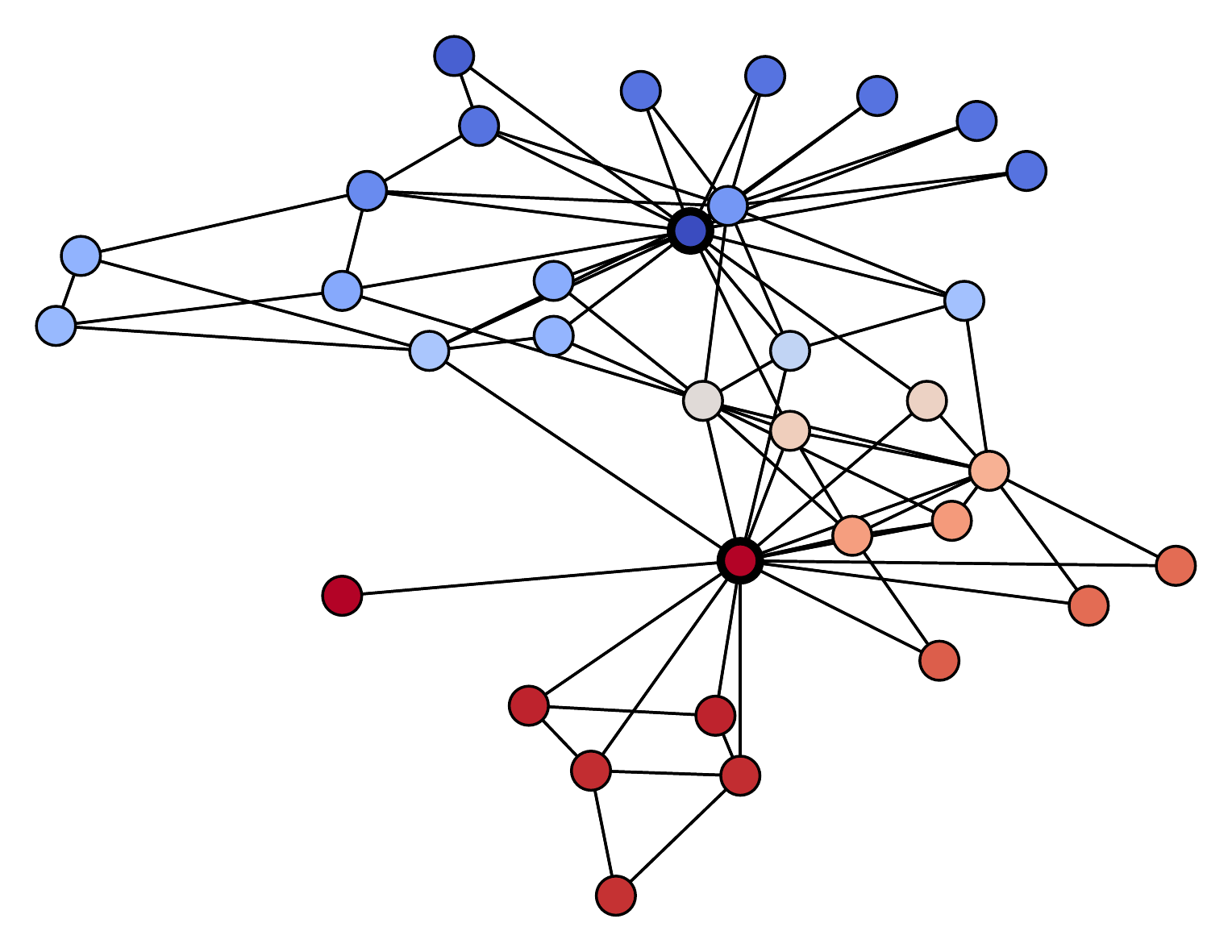}}
    \caption{Binary classification of the Karate Club graph  \citep{zachary1977information} with 2 seeds (indicated with a black circle). Red nodes have label 1, blue nodes have label 2.}
    \label{fig:karate}
\end{figure}

A natural approach, proposed by \cite{zhu2003semi}, consists in assigning label 1 to all nodes with temperature above 0.5 and label 2 to other nodes. The analysis of section \ref{sec:model} suggests that it is preferable to set the threshold to the mean temperature, 
$$\bar T = \frac 1 n \sum_{i=1}^n T_i.
$$
Specifically, all nodes with temperature above $\bar T$ are assigned  label 1, the other are assigned  label 2. Equivalently, temperatures are centered before classification: after centering, nodes with positive temperature are  assigned  label 1, the others are assigned  label 2.

Note that the temperature of each node can be used to assess the confidence in the classification: the closer the temperature to  the mean, the lower the confidence.
This is illustrated by Figure \ref{fig:karate} (the lighter the color, the lower the confidence). In this case, only one node is misclassified and has indeed a temperature close to the mean.

\subsection{Multi-class classification}
In the presence of more than 2 labels, we use 
 a \textit{one-against-all} strategy: the seeds of each label alternately serve as hot sources (temperature 1) while all the other seeds serve as cold sources (temperature 0). After centering the temperatures (so that the mean temperature of each diffusion is equal to 0), each node is assigned  the label that maximizes its temperature. This algorithm, we refer to as Dirichlet classifier, is parameter-free.

\begin{algorithm}[ht]
\caption{Dirichlet classifier}
\begin{algorithmic}[1]
\REQUIRE Seed set $S$ and associated labels $y\in\{1,\ldots,K\}$.
\FOR{$k$ in $\{1,\ldots,K\}$}
    \STATE $T^S = 0$.
    \FOR{$i \in S$}
        \IF{$y_i = k$}
            \STATE $T^S_i = 1$.
        \ENDIF
    \ENDFOR
    \STATE $T^{(k)} \leftarrow \text{Dirichlet}(S, T^S)$.
    \STATE $\Delta^{(k)} \leftarrow T^{(k)} - \text{mean}({T}^{(k)})$
\ENDFOR
    \FOR{$i \not \in S$}
\STATE $x_i = \argmax_{k=1,\ldots,K}(\Delta^{(k)}_i)$
    \ENDFOR
\RETURN $x$, labels of nodes outside $S$
\end{algorithmic}
\label{algo:maxdiff}
\end{algorithm}

The key difference with the {\it vanilla} method lies in temperature centering (line 9 of the algorithm). Another variant proposed by \cite{zhu2003semi} consists in rescaling   the temperature vector  by the  weight of the considered label in the seeds (see equation (9) in their paper).

\subsection{Time complexity}

The time complexity depends on the algorithm used to solve the Dirichlet problem. We here focus on the approximate solution by successive iterations of \eqref{eq:linear}.
Let $m$ be the number of edges of the graph. Using the  Compressed Sparse Row format for the adjacency matrix, each matrix-vector product  has a complexity of $O(m)$. The complexity of Algorithm \ref{algo:maxdiff} is then $O(NKm)$, where $N$ is the number of iterations.  Note that the $K$ Dirichlet problems are independent and can thus be computed in parallel.

\section{Analysis}
\label{sec:model}

In this section, we prove the consistency of Algorithm \ref{algo:maxdiff} on a simple block model. In particular, we highlight the importance of temperature centering in the analysis.

\subsection{Block model}

Consider a graph of $n$ nodes consisting of $K$ blocks  of respective sizes $n_1,\ldots,n_K$, forming a partition of the set of nodes.  There are  $s_1,\ldots,s_K$ seeds in these blocks, which are respectively  assigned labels $1,\ldots,K$. Intra-block edges have weight $p$ and inter-block edges have weight $q$. 
We expect  the algorithm to assign label $k$ to all nodes of  block $k$ whenever $p>q$, for all $k=1,\ldots,K$. 

\subsection{Dirichlet problem}

Consider  the Dirichlet problem when  the temperature of the $s_1$ seeds of block 1 is set to 1 and the temperature  of the other seeds is set to 0. We have an explicit  solution to this Dirichlet problem, whose proof is provided in the appendix.

\begin{lemma}
\label{prop:dirichlet}
Let $T_k$ be the temperature of non-seed nodes of  block $k$ at equilibrium. We have:
    \begin{align*}
(s_1(p-q) + nq) T_1  &= s_1 (p-q) + n\bar T q,\\
(s_k(p-q) + nq) T_k  &= n\bar T q\quad \quad k=2,\ldots,K,
\end{align*}
where $\bar T$ is the average temperature, given by:
$$
 \bar T = \left(\frac{s_1}{n} \frac{n_1(p-q) + nq}{ s_1(p-q) + nq}\right) / \left(1-\sum_{k=1}^K \frac{(n_k - s_k)q}{s_k(p-q)+nq}\right).
$$
\end{lemma}

\subsection{Classification}
\label{ssec:clf}

We now state the main result of the paper: the Dirichlet classifier is a consistent algorithm for the block model, in the sense that all nodes are correctly classified whenever $p> q$. 

\begin{theorem}
If $p> q$, then $x_i=k$ for all non-seed nodes $i$ of each block $k$,  for any parameters $n_1,\ldots,n_K$ (block sizes) and $s_1,\ldots,s_K$ (numbers of seeds).
\end{theorem}

\begin{proof}
Let $\delta^{(1)}_k = T_k - \bar T$ be the deviation of temperature of non-seed nodes of block $k$ for the Dirichlet problem associated with label 1.
    In view of Lemma \ref{prop:dirichlet}, we have:
\begin{align*}
(s_1(p-q) + nq) \delta^{(1)}_1  &= s_1 (p-q) (1-\bar T),\\
(s_k(p-q) + nq)\delta^{(1)}_k  &= -s_k(p-q) \bar T \quad \quad k=2,\ldots,K,
\end{align*}
For $p>q$, using the fact that $\bar T \in (0,1)$, we get $\delta^{(1)}_1 > 0$ and $\delta^{(1)}_k<0$ for all $k=2,\ldots,K$. By symmetry, for each label $l = 1,\ldots,K$,
$\delta^{(l)}_l > 0$ and $\delta^{(l)}_k<0$ for all $k\ne l$.
We deduce that for each block $k$, $x_i=\arg\max_{l}\delta^{(l)}_k = k$ for all non-seed nodes $i$ of  block $k$.
\end{proof}

Observe that the temperature centering is critical for consistency. In the absence of centering,  non-seed nodes of block 1 are correctly classified if and only if their temperature is the highest in the Dirichlet problem associated with label 1.
In view of Lemma \ref{prop:dirichlet}, this means that for all $k=2,\ldots,K$,
\begin{align*}
&s_1 q \frac{n_1(p-q) + nq}{s_1(p-q) + nq} + s_1(p-q)\left(1-\sum_{j=1}^K \frac{(n_j - s_j)q}{s_j(p-q)+nq}\right)
> s_k q \frac{n_k(p-q) + nq}{s_k(p-q) + nq}.
\end{align*}
This condition might be violated even if $p>q$, depending on the parameters  $n_1,\ldots,n_K$ and $s_1,\ldots,s_K$. 
In the practically interesting case where $s_1 << n_1,\ldots,s_K << n_K$ for instance (low fractions of seeds), the condition requires:
$$
s_1  (n_1(p-q) + nq) > s_k  (n_k(p-q) + nq).
$$
For blocks of same size, this means that only blocks with the largest number of seeds are correctly classified.  The classifier is biased towards labels with a large number of seeds. This sensitivity of the {\it vanilla} algorithm to the  label distribution  of seeds will be confirmed in the experiments on real graphs.

\section{Experiments}
\label{sec:exp}

In this section, we show the impact of temperature centering on the quality of classification using both synthetic and real data.
We do {\it not} provide a general benchmark of classification methods as the focus of the paper is on the impact of temperature centering in heat diffusion methods. 
We focus on 3 algorithms: the vanilla algorithm (without temperature centering), the weighted version proposed by \citep{zhu2003semi} (also without temperature centering) and our algorithm (with temperature centering).

All datasets and codes are available online\footnote{\url{https://github.com/nathandelara/Dirichlet}}, making the experiments fully reproducible.


\subsection{Synthetic data}

We first use the stochastic block model (SBM) \citep{airoldi2008mixed} to generate graphs with an underlying structure in clusters. This is the stochastic version of the block model used in the analysis. There are $K$ blocks of respective sizes $n_1,\ldots,n_K$. Nodes of the same block are connected with probability  $p$ while nodes in different blocks are connected probability $q$.
We denote by $s_k$ the number of  seeds in block $k$ and by $s$  the total number of seeds.

We first compare the performance of the algorithms on a binary classification task ($K=2$) for a graph of $n=10\,000$ nodes with $p=10^{-3}$ and $q=10^{-4}$, in two different settings:
\begin{itemize}
{\bf \item Seed asymmetry:} Both blocks have the same size   $n_1= n_2 = 5000$ but different numbers of seeds, with $s_1/s_2 \in \{1, 2, \dots,10\}$ and   $s_2 = 250$ (5\% of nodes in block 2).
\item 
{\bf Block size asymmetry:} The blocks have different sizes with 
ratio
$n_1/n_2 \in \{1, 2, \dots, 10\}$ and seeds in proportion to these sizes, with a total of $s= 1\,000$ seeds ($10\%$ of nodes). 
\end{itemize}

For each configuration, the experiment is repeated 10 times. Randomness comes both from the generation of the graph and from the selection of the seeds. We report the F1-scores in Figure \ref{fig:sbm} (average $\pm$ standard deviation). Observe that the variability of the results is very low due to the relatively large size of the graph. As expected, the centered version is much more robust to both types of asymmetry. Besides, in case of asymmetry in the  seeds, the weighted version of the algorithm tends to amplify the bias and leads to lower scores  than the vanilla version.

\begin{figure}[h]
    \centering
    \subfloat[Seed asymmetry.]{\includegraphics[width=0.49\linewidth]{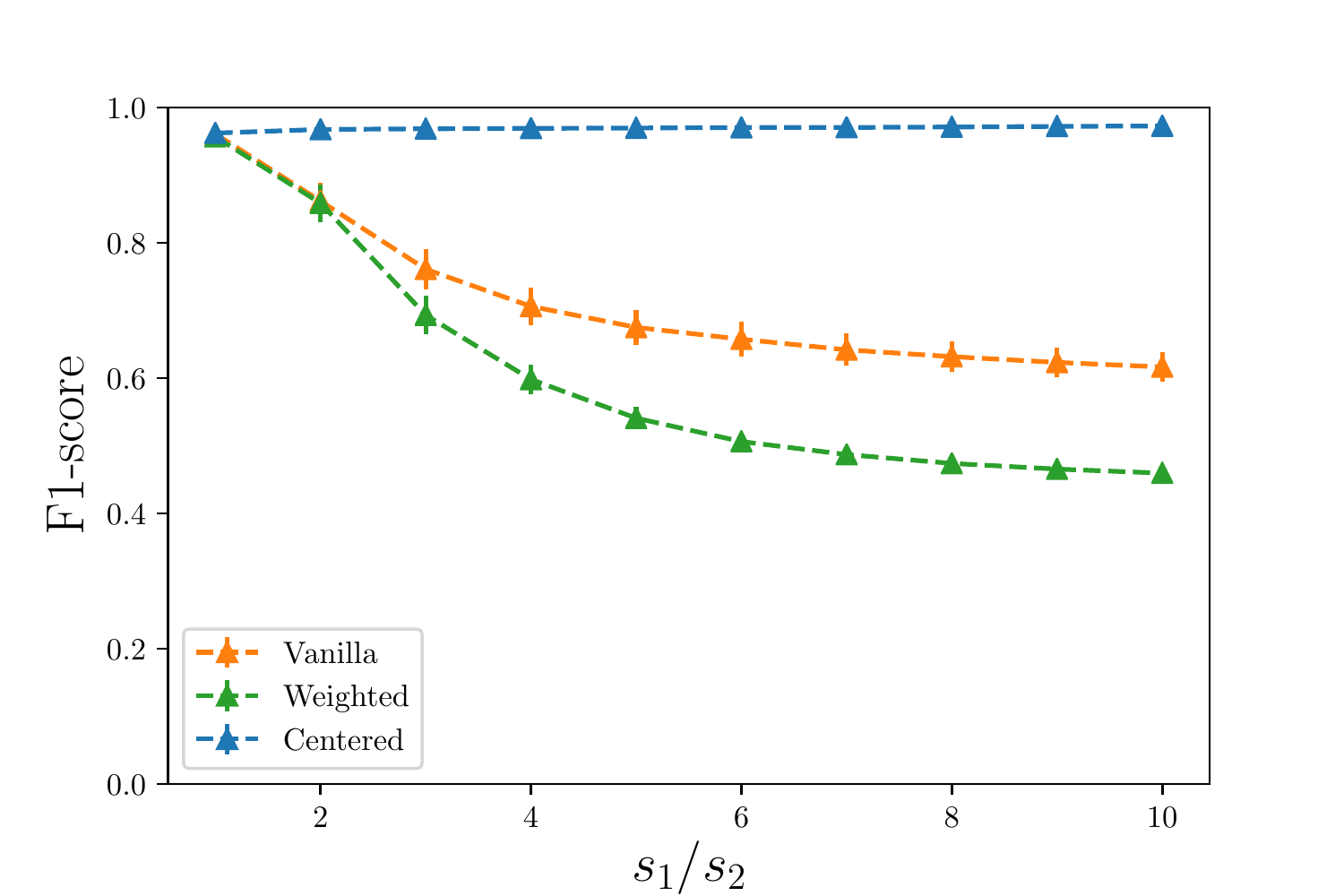}}
    \subfloat[Block size asymmetry.]{\includegraphics[width=0.49\linewidth]{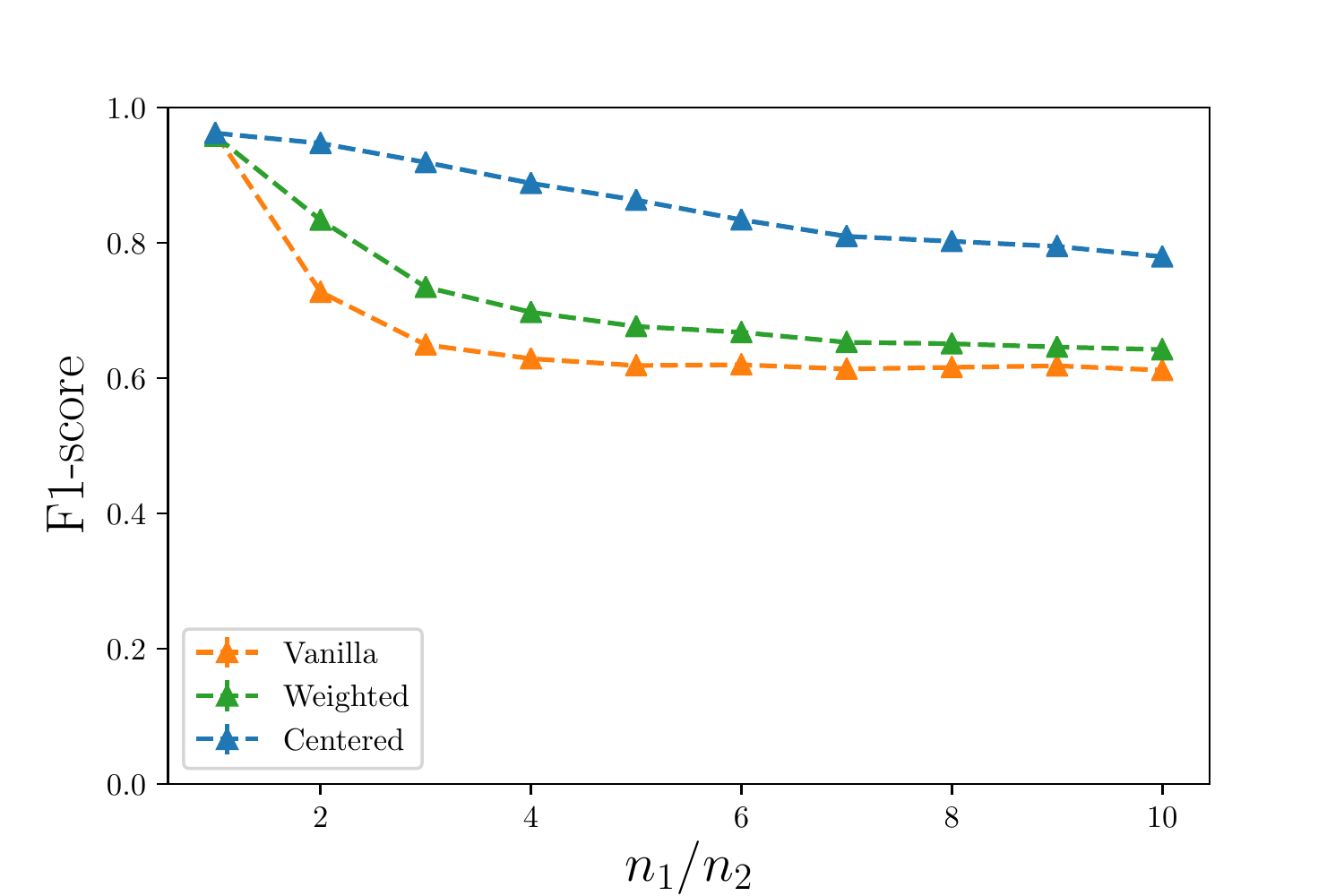}}
    \caption{Binary classification performance on the SBM.}
    \label{fig:sbm}
\end{figure}

We show in Figure \ref{fig:sbm/multi} the same type of results for $K = 10$ blocks and $p=5.10^{-2}$. For the block size asymmetry, the  size of blocks $2,\ldots,10$ is set to $1\,000$.

\begin{figure}[h]
    \centering
    \subfloat[Seed asymmetry.]{\includegraphics[width=0.49\linewidth]{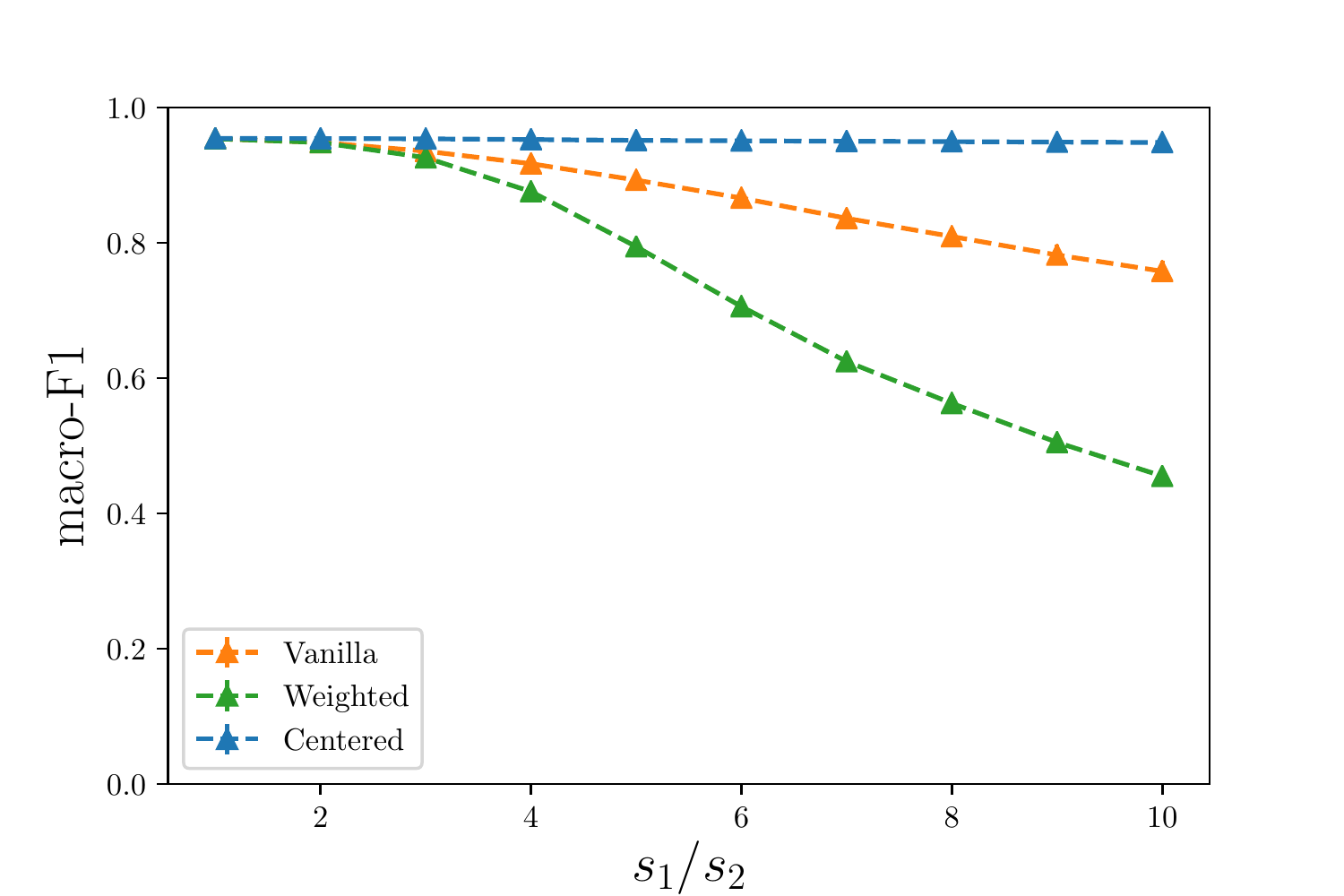}}
    \subfloat[Block size asymmetry.]{\includegraphics[width=0.49\linewidth]{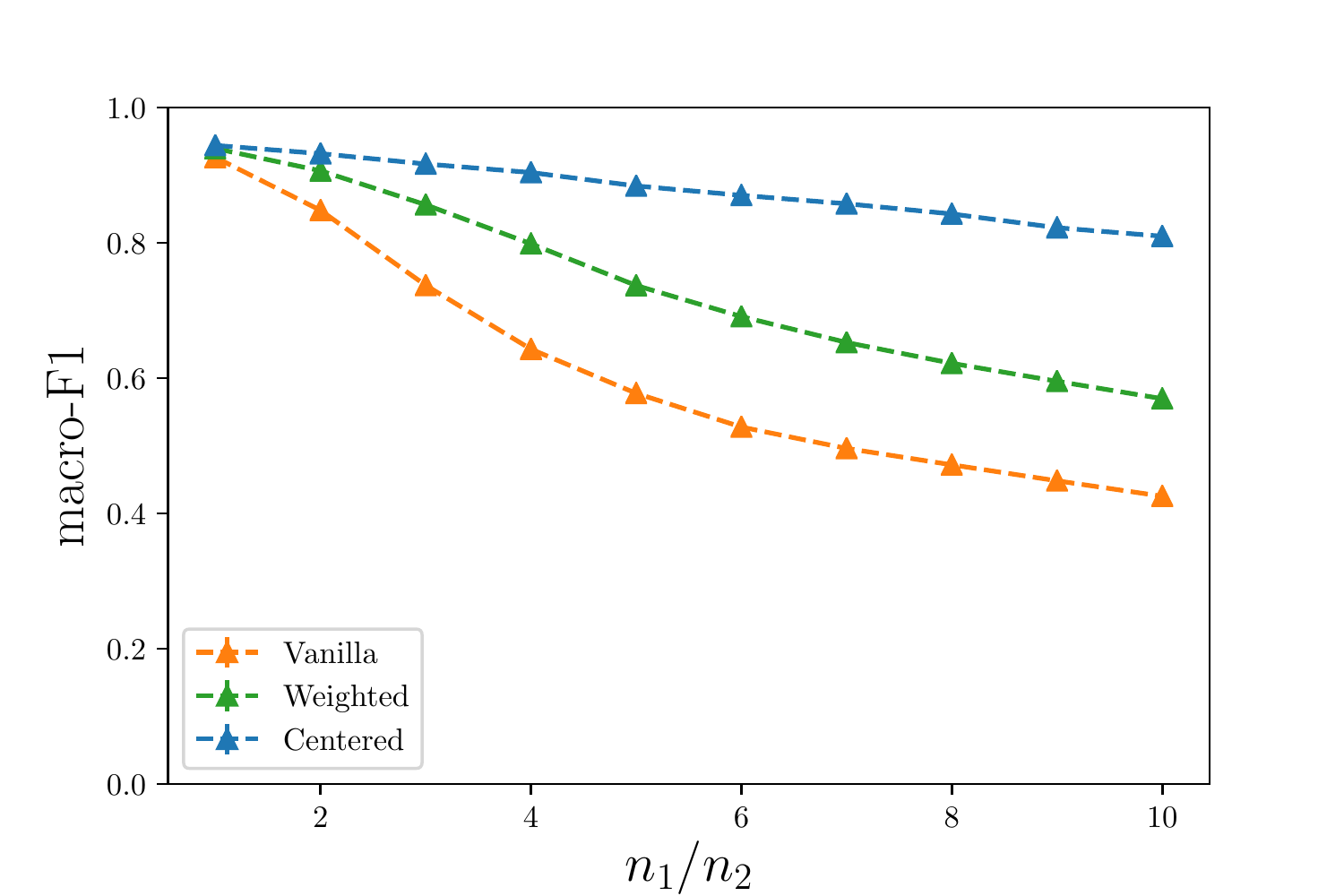}}
    \caption{Multi-label classification performance on the SBM.}
    \label{fig:sbm/multi}
\end{figure}

\subsection{Real data}
\label{ssec:datasets}

We  use datasets from the NetSet\footnote{\url{https://netset.telecom-paris.fr/}} and SNAP\footnote{\url{https://snap.stanford.edu/data/}} collections (see 
Table \ref{tab:datasets}).

\begin{table}[ht]
    \centering
    \caption{Datasets.}
    \begin{tabular}{l|cccc}
        dataset & \# nodes & \# edges & \# classes & $(\%)$ labeled\\
        \toprule
        CO & $2\ 708$ & $10\ 556$ & 7 & 100\\ 
        CS & $3\ 264$ & $9\ 072$ & 6 & 100\\
        WS & $4\ 403$ & $112\ 834$ & 16 & 100\\
        WV & $10\ 012$ & $792\ 091$ & 11 & 100\\
        WL & $3\ 210\ 346$ & $67\ 196\ 296$ & 10 & 100\\
        DBLP & $317\ 080$ & $2\ 099\ 732$ & 5000 & 29\\
        Amazon & $334\ 863$ & $1\ 851\ 744$ & 5000 & 5
    \end{tabular}
    \label{tab:datasets}
\end{table}

These datasets can be categorized into 3 groups:
\begin{itemize}
\item{\bf Citations networks:} Cora (CO) and CiteSeer (CS) are citation networks between scientific publications. These are standard  datasets for node classification \citep{fey2018splinecnn, huang2018adaptive, wijesinghe2019dfnets}. 

\item{\bf Wikipedia graphs:} Wikipedia for schools (WS) \citep{haruechaiyasak2008article}, Wikipedia vitals (WV) and Wikilinks (WL) are graphs of hyperlinks between different selections of Wikipedia pages. In WS and WV, pages are labeled by category (People, History, Geography...). For WL, pages are labeled through clusters of  words used in these articles.
As these graphs are directed, we use the extension of the algorithm described in \S\ref{ssec:ext}, with nodes considered as heat sources.

\item{\bf Social networks:} DBLP and Amazon are social networks with partial ground-truth communities \citep{snapnets}. As nodes are partially labeled and some nodes have  several labels,  the results for these datasets are presented separately, with specific  experiments based on binary classification.
\end{itemize}




For the citation networks and the Wikipedia graphs, we  compare the classification performance of the algorithms in terms of  macro-F1 score and two seeding policies:
\begin{itemize}
\item \textbf{Uniform sampling}, where seeds are sampled uniformly at random.
\item \textbf{Degree sampling}, where seeds are sampled in  proportion to their degrees.
\end{itemize}

In both cases, the seeds represent $1\%$  of the total number of nodes in the graph.
The process is repeated 10 times for each configuration. We defer the results for the weighted version of the algorithm to the supplementary material as they are very close to those obtained with the vanilla algorithm.

We report the results in Tables~\ref{tab:macrof1/node} and \ref{tab:macrof1/edge} for uniform sampling  and degree sampling, respectively. We see that centered version outperforms the vanilla one by a significant margin.


\begin{table}[ht]
    \centering
    \caption{Macro-F1 scores (mean $\pm$ standard deviation) with uniform seed sampling.}
    \vspace{1mm}
    
    \begin{tabular}{l|ccccc}
        algorithm & CO & CS & WS & WV & WL\\
        \toprule
        Vanilla & $0.19 \pm 0.12$ & $0.17 \pm 0.04$ & $0.04 \pm 0.02$ & $0.09 \pm 0.04$ & $0.19 \pm 0.01$\\
        Centered & $\mathbf{0.42 \pm 0.18}$ & $\mathbf{0.36 \pm 0.04}$ & $\mathbf{0.16 \pm 0.11}$ & $\mathbf{0.55 \pm 0.03}$ & $\mathbf{0.51 \pm 0.01}$\\
    \end{tabular}
    \label{tab:macrof1/node}
\end{table}

\begin{table}[ht]
    \centering
    \caption{Macro-F1 scores (mean $\pm$ standard deviation) with degree seed sampling.}
    \vspace{1mm}
    
    \begin{tabular}{l|ccccc}
        algorithm & CO & CS & WS & WV & WL\\
        \toprule
        Vanilla & $0.30 \pm 0.08$ & $0.16 \pm 0.07$ & $0.02 \pm 0.01$ & $0.10 \pm 0.04$ & $0.34 \pm 0.00$\\
        Centered & $\mathbf{0.51 \pm 0.08}$ & $\mathbf{0.26 \pm 0.13}$ & $\mathbf{0.06 \pm 0.04}$ & $\mathbf{0.48 \pm 0.02}$ & $\mathbf{0.45 \pm 0.00}$\\
    \end{tabular}
    \label{tab:macrof1/edge}
\end{table}

For the social networks, we perform independent binary classifications for each of the 3 dominant labels and average the scores. As these datasets have only a few labeled nodes, we consider the most favorable scenario where   seeds are sampled in proportion to the labels.
Seeds  still represent $1\%$ of the nodes. The results are shown in Table \ref{tab:macrof1/social}.

\begin{table}[ht]
    \centering
    \caption{Macro-F1 scores (mean $\pm$ standard deviation) with balanced seed sampling.}
    \vspace{1mm}
    
    \begin{tabular}{l|cc}
        algorithm & DBLP & Amazon \\
        \toprule
        Vanilla & $0.04 \pm 0.00$ & $0.05 \pm 0.01$\\
        Centered & $\mathbf{0.19 \pm 0.01}$ & $\mathbf{0.18 \pm 0.02}$\\
    \end{tabular}
    \label{tab:macrof1/social}
\end{table}

Finally, we assess the classification performance of the algorithms in the case of seed asymmetry. Specifically, we first sample $s = 1\%$ of the nodes uniformly at random and  progressively increase the number of seeds for the dominant class of each dataset, say label 1.

The process is repeated 10 times for each configuration. Figure~\ref{fig:seed} shows the  macro-F1 scores. We see that the performance of the centered algorithm remains steady in the presence of seed asymmetry.

\begin{figure}[h]
    \centering
    \includegraphics[width=0.49\linewidth]{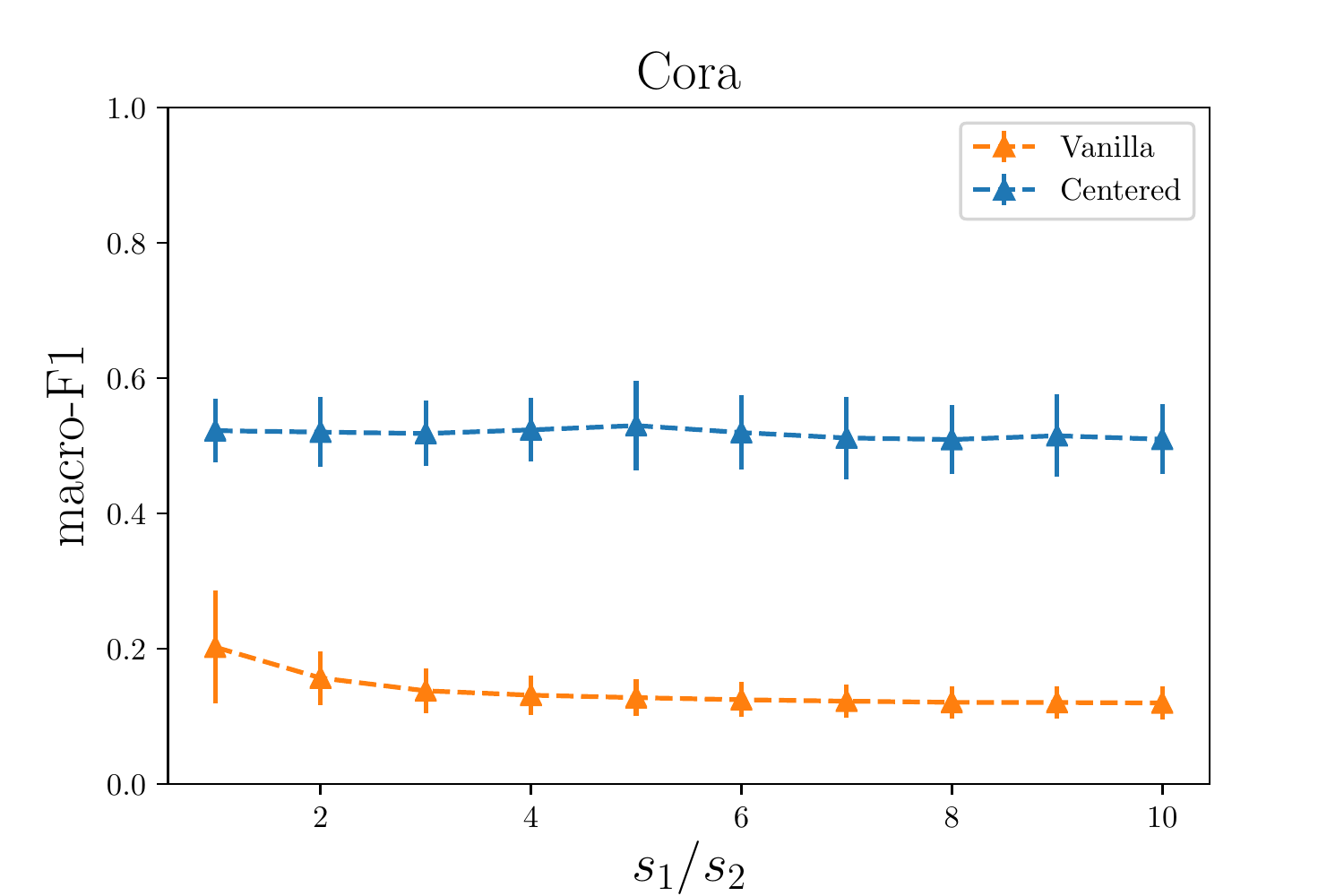}
    \includegraphics[width=0.49\linewidth]{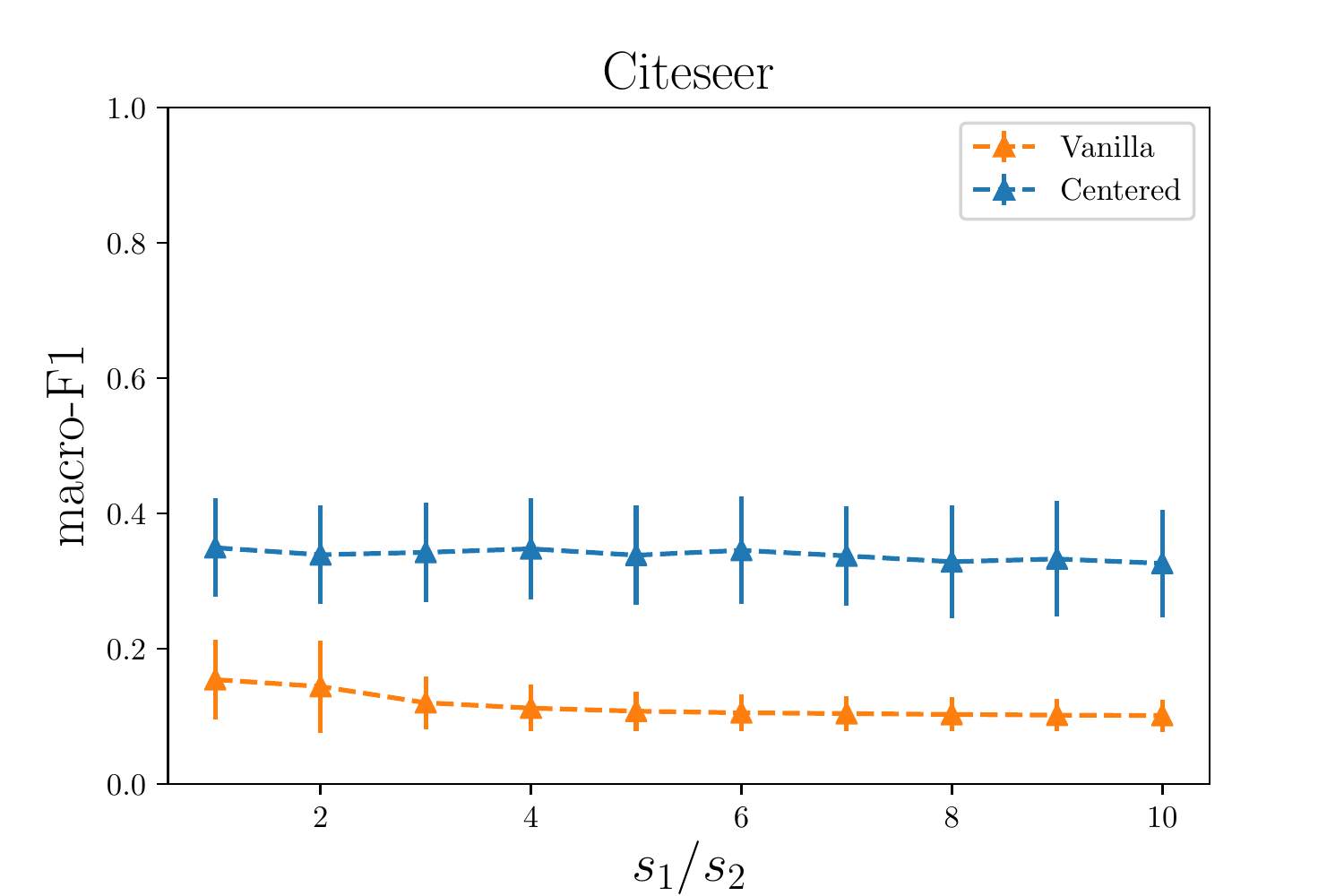}
    \includegraphics[width=0.49\linewidth]{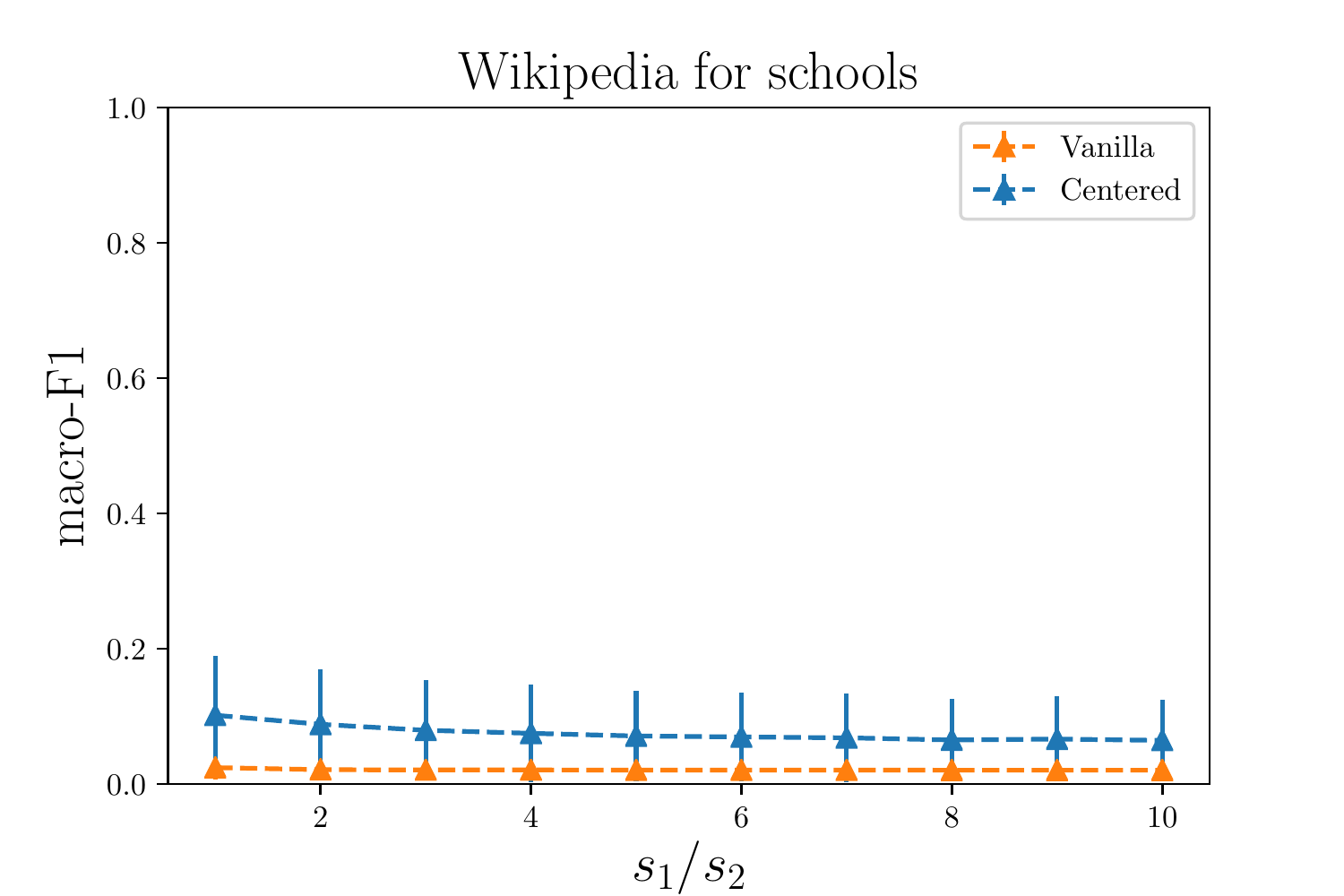}
    \includegraphics[width=0.49\linewidth]{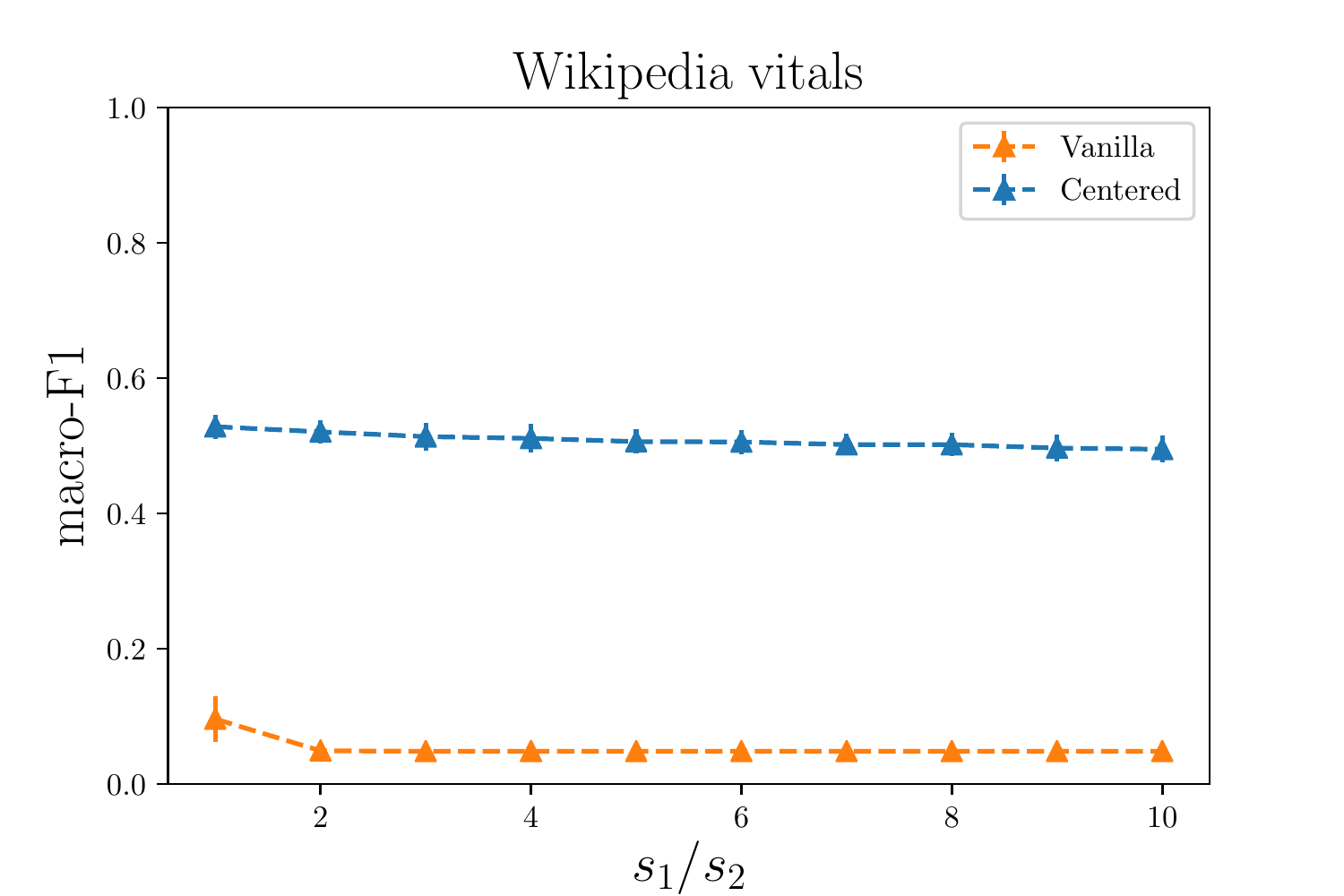}
    \caption{Impact of the fraction of seeds for label 1  on macro-F1 score (mean $\pm$ standard deviation).}
    \label{fig:seed}
\end{figure}

\section{Conclusion}
\label{sec:conc}

We have proposed a novel approach to node classification based on heat diffusion. Specifically, we propose to center the temperatures of each solution to the  Dirichlet problem before classification. We have proved the consistency of this algorithm on a simple block model and we have shown  that it drastically improves classification performance on real datasets with respect to the vanilla version.

In future work, we plan to extend this algorithm to soft classification, using the centered temperatures to get a confidence score for each node of the graph. Another interesting research perspective is to extend our proof of consistency of the algorithm to {\it stochastic} block models.


\bibliographystyle{named}
\bibliography{biblio}

\appendix

\section*{Appendix: Proof of Lemma \ref{prop:dirichlet}}

\begin{proof}
In view of (2), we have:
\begin{align*}
&(n_1(p-q) + nq) T_1 = s_1 p + (n_1-s_1)pT_1 + \sum_{j\ne 1} (n_j - s_j) qT_j,\\
&(n_k(p-q) + nq) T_k = s_1 q + (n_k-s_k)pT_k + \sum_{j\ne k} (n_j - s_j) qT_j,\quad k=2,\ldots,K.
\end{align*}
We deduce:
\begin{align*}
(s_1(p-q) + nq) T_1  &= s_1 p + Vq,\\
(s_k(p-q) + nq) T_k  &= s_1 q + Vq\quad \quad k=2,\ldots,K.
\end{align*}
with 
$$
V = \sum_{j=1}^K  (n_j - s_j) T_j.
$$
The proof then follows from the fact that 
$$
n \bar T =  s_1 + \sum_{j=1}^K  (n_j - s_j) T_j = s_1 + V.
$$
\end{proof}

\end{document}